\documentclass[letterpaper]{article} 
\usepackage{aaai23}  
\usepackage{times}  
\usepackage{helvet}  
\usepackage{courier}  
\usepackage[hyphens]{url}  
\usepackage{graphicx} 
\usepackage{natbib}  
\usepackage{caption} 
\usepackage{algorithm}
\usepackage{algorithmic}
\usepackage{amsmath}
\usepackage{amsfonts}
\usepackage{thmtools,thm-restate}
\usepackage{amsthm}

\newcommand\numberthis{\addtocounter{equation}{1}\tag{\theequation}}

\usepackage{newfloat}
\usepackage{listings}
\DeclareCaptionStyle{ruled}{labelfont=normalfont,labelsep=colon,strut=off} 
\floatstyle{ruled}
\newfloat{listing}{tb}{lst}{}
\floatname{listing}{Listing}
\title{Decision-Focused Evaluation: Analyzing Performance of Deployed Restless Multi-Arm Bandits}
\author {
    Paritosh Verma\textsuperscript{\rm 1,2}\thanks{Work was done during an internship at Google Research India.},
    Shresth Verma \textsuperscript{\rm 1},
    Aditya Mate \textsuperscript{\rm 1,3}$^*$,
    Aparna Taneja \textsuperscript{\rm 1},
    Milind Tambe \textsuperscript{\rm 1}
}
\affiliations {
    \textsuperscript{\rm 1} Google Research India\\
    \textsuperscript{\rm 2} Purdue University\\
    \textsuperscript{\rm 3} Harvard University\\
    verma136@purdue.edu, vermashresth@google.com, aditya\_mate@g.harvard.edu,
    \{aparnataneja, milindtambe\}@google.com
}

\usepackage{bibentry}

\begin{document}

\maketitle

\begin{abstract}




Restless multi-arm bandits (RMABs) is a popular decision-theoretic framework that has been used to model real-world sequential decision making problems in public health, wildlife conservation, communication systems, and beyond. Deployed RMAB systems typically operate in two stages: the first predicts the unknown parameters defining the RMAB instance, and the second employs an optimization algorithm to solve the constructed RMAB instance.

In this work we provide and analyze the results from a first-of-its-kind deployment of an RMAB system in public health domain, aimed at improving maternal and child health. Our analysis is focused towards understanding the relationship between prediction accuracy and overall performance of deployed RMAB systems. This is crucial for determining the value of investing in improving predictive accuracy towards improving the final system performance, and is useful for diagnosing, monitoring deployed RMAB systems. 

Using real-world data from our deployed RMAB system, we demonstrate that an improvement in overall prediction accuracy may even be accompanied by a degradation in the performance of RMAB system -- a broad investment of resources to improve overall prediction accuracy may not yield expected results. Following this, we develop decision-focused evaluation metrics to evaluate the predictive component and show that it is better at explaining (both empirically and theoretically) the overall performance of a deployed RMAB system.

\end{abstract}


\section{Introduction}  

Restless Multi Armed Bandits (RMABs) is a general framework for solving sequential decision making problems and has been employed in a wide variety of application domains such as planning preventive interventions for healthcare ~\cite{mate2022field}, anti-poaching patrols \cite{QianRestlessPoachers}, communication systems~\cite{Liu2010,liu2012learning}, sensor monitoring tasks~\cite{glazebrook2006indexable}, etc.
Most of the works on RMAB have focused on studying the optimization problem of allocating limited resources, assuming perfect knowledge of the underlying parameters of the RMAB model \cite{bertsimas2000restless,nino2001restless,verloop2016asymptotically,mate2020collapsing}. As a result, RMABs have seen limited deployment in practice, especially in applications such as healthcare and conservation where the RMAB parameters of the agents being catered to, are unknown in the real-world. 
Existing approaches that implement RMAB solutions to resource allocation problems typically adopt a two-staged, predict-then-optimize framework \cite{wang2020restless,osband2013more,jung2019regret,mate2022field}. In the first stage, these approaches learn a machine learning model that predicts the necessary RMAB parameters and then in the second stage, solve the RMAB optimization problem using these predictions. 

We pose the question of understanding the relationship between prediction accuracy and overall RMAB system performance. Such an understanding is important for two reasons. First, we want to understand if investing in improved prediction accuracy warrants improved system performance. Second, once an RMAB-based system is deployed in the real-world, we are interested in monitoring its performance and providing diagnosis to understand the potential sources of improvements in the data-to-deployment pipeline.  In general, we expect that if the RMAB model parameters are accurately predicted, the system's decisions are indeed guaranteed to be optimal; similarly, we expect that given a fixed optimization engine, higher overall prediction accuracy would lead to improved RMAB performance.


\paragraph{Our contributions.} \emph{Our first contribution is to show, using for the first time an RMAB system deployed in the public health domain, that improving machine learning prediction accuracy alone -- particularly measured using standard measures of error like RMSE or MAE -- may not lead to improved overall system performance. In fact there may be a degradation in performance.} Evaluating the performance of RMABs in context of a real-world deployment (described below), we demonstrate this important phenomenon, and as our second contribution propose an alternative, decision-focused evaluation approach of the machine learning component to address this issue. Broadly speaking, a key take away lesson of our work is that instead of investing resources including compute, data or human resources for broad improvements in prediction accuracy, our proposed decision-focused evaluation metrics may provide a better guide for investment in RMAB deployment. Moreover, many systems deployed in the real-world follow a two-staged, predict-then-optimize framework \cite{ford2015beware,fang2016deploying,perrault2019ai}, our work highlights that even in these domains we shouldn't directly assume, or design systems based on a correlation between prediction accuracy and overall performance. 

We demonstrate our results and describe the methodological contributions via one such application of the predict-then-optimize framework for RMABs. 
We collaborate with ARMMAN, an Indian NGO, that aims to improve access to maternal health information in underprivileged communities. Through their flagship program mMitra, ARMMAN delivers critical health information to new and expectant mothers via automated phone calls. However, the engagement rates among mMitra's beneficiaries dwindle over time and as a fix, ARMMAN delivers live service calls to encourage engagement. Due to limited resources, only a small fraction of beneficiaries can be selected for live service calls every week. This is cast as an RMAB problem where we must decide which beneficiaries to choose every week for live service calls. 

Using our proposed decision-focused evaluation approach for RMABs, we analyze the performance of RMAB-based system deployed for ARMMAN, i.e., all our analysis is based on real-world data where decisions concerning real individuals were taken via an RMAB. The deployed RMAB system employs the Whittle index based method ---  which is the most prominent solution concept for RMABs. We compare and contrast different methods of defining errors in top-k Whittle indices. As our final contribution in this paper, we show that an error definition based on the Spearman's footrule measure \cite{diaconis1977spearman}\footnote{Spearman's footrule is used to quantify the disarray between two permutations} is best suited in this context for decision-focused evaluation. Our proposed approach is indeed able to predict the real-world performance of our RMAB system better than prediction accuracy analysis.

\section{Related Work}
Sequential resource allocation problems arise in many real-world scenarios in healthcare domain.
For example, adherence monitoring is an extensively studied problem \cite{martin2005challenge} where the goal is to carefully allocate the limited number of available healthcare workers or resources to monitor and improve patients' adherence to medication for diseases like cardiac problems~\cite{corotto2013heart}, tuberculosis~\cite{Killian_2019, ong2014effects,chang2013house} and HIV~\cite{HIV}. These mentioned works largely focus on developing machine learning model to classify beneficiaries as high risk, or predict their future adherence patterns. However, these approaches essentially rely on making one-shot predictions and fail to capture the sequential aspect the of decision making needed to maximize long term rewards. Other works have also used reinforcement learning to design health monitors and provide personalized suggestions and notifications to users~\cite{liao2020personalized, pollack2002pearl}; notably, these works do not deal with the problem of allocating limited resources, as there is no constraint on the number of notifications being sent.

Restless Multi-Armed Bandit is a popular framework \cite{whittle-rbs,jung2019regret} used for solving sequential resource allocation problems that require long term planning. 
In an RMAB instance the decision choices/alternatives are represented by Markov Decision Processes (MDPs) which are in turn characterized by their transition dynamics.
One major challenge in using RMABs in the real-world is the problem of unknown transition dynamics. Several previous works such as  \cite{LiuIndexRB,QianRestlessPoachers,mate2021risk} assume that transition dynamics are already known beforehand, making them unsuitable for real-world deployment.



A common approach is to estimate the transition dynamics by inferring them using background information.
However, the predictive model learnt in such works maximizes the accuracy of predicting transition dynamics. This can create a mismatch between the objective being maximized and the final decision outcome. Such a mismatch can result in unintended consequences for the system, as highlighted in \cite{Boettiger_2022} in the context of fisheries management.

In contrast, decision-focused learning is a line of work wherein the decision outcomes are directly optimized rather than following a predict-then-optimize framework. Different kinds of one-shot \cite{donti2017task,perrault2020end,wilder2019melding} and sequential optimization problems ~\cite{wang2021learning,futoma2020popcorn} can be solved end-to-end by blending the decision outcome into the downstream optimization problem. In \cite{wang2022decision}, a decision focused learning framework is proposed for RMABs, where transition dynamics are learnt by directly optimizing the final decision outcome using off-policy policy evaluation (OPE).  Unfortunately, OPE with limited data is often not very stable and requires extensive tuning to get desired results \cite{huang2020importance}. Moreover such an end-to-end optimization results in low interpretability of the results. There are several differences between this line of work and ours. We focus on real-world deployed system to illustrate that improved prediction accuracy using standard error metrics may not result in improved overall decision quality. In addition we focus on RMABs and provide tailored approaches to evaluate and explain the performance of the prediction component of RMABs.



\section{Preliminaries}

\paragraph{Restless Multi-Armed Bandits}
The RMAB framework is characterized by $N$ independent Markov Decision Processes \cite{mdp-puterman}, which are referred to as arms. Each arm is represented by a 4-tuple $\{\mathcal{S}, \mathcal{A}, R, \mathcal{P}\}$. $\mathcal{S}$ denotes the state space,  which could be a good state "beneficiary adhering to the program" or bad state "beneficiary not adhering to the program". $\mathcal{A}$ is the set of possible actions which we consider to be binary in our case, i.e., an action could be active, pulling an arm; or passive, not pulling an arm.  $R$ is the reward function $R: \mathcal{S} \times \mathcal{A} \times \mathcal{S} \rightarrow \mathbb{R}$. And $P$ denotes the probability of transitioning to the next state $s' \in \mathcal{S}$ starting from a current state $s \in \mathcal{S}$ under action $a \in \mathcal{A}$. We denote this probability as $P(s, a, s')$. The policy $\pi$ for an arm is defined as the mapping $\pi : \mathcal{S} \rightarrow \mathcal{A}$, i.e., it dictates the action to be taken given the current state. The objective that we maximize in the RMAB framework is sum of expected discounted rewards for all arms. For a single arm having a starting state $s_0$, this reward can be written as $V_{\beta}^\pi(s_0) = \mathbb{E}\left[\sum_{t=0}^\infty \beta^tR(s_t, \pi (s_t), s_{t+1}|\pi, s_0)\right]$. The next state are drawn from the distribution $s_{t+1} \sim P(s_{t}, {\pi(s_t)}, s_{t+1})$ where $\beta \in [0,1)$ is called the discount factor and $P$ represents the transition probabilities of that arm.

Finding the optimal solution for RMAB problems is known to be PSPACE-hard \cite{papadimitriou1994complexity}. The Whittle Index policy \cite{whittle-rbs} is a computationally efficient heuristic for solving RMAB. The idea of Whittle Index is to provide a subsidy whenever the passive action is chosen by the planner. The value of the infimum subsidy such that there is no difference in choosing among active or passive actions is then defined as the Whittle Index. Specifically, $W(s)=inf_{\lambda}\{\lambda:Q_{\lambda}(s,a=0)=Q_{\lambda}(s,a=1)\}$ where $Q(s, a)$ is the $Q$-value or expected discounted future reward of taking an action $b$ from state $s$. The Whittle index policy operates by selecting $k$ beneficiaries having the highest Whittle indices in each decision step.

\paragraph{ARMMAN}
Beneficiaries are enrolled into ARMMAN's mMitra program by healthcare workers. These enrollments are either made at hospitals or through door-to-door surveys. At registration time, beneficiaries' socio-demographic information such as age, education, income, number of children, gestational age, etc is noted. Additionally, based on whether the beneficiaries have already delivered the baby or not, they are enrolled in the Antenatal Care or Postnatal Care program. Then automated voice calls with health information tailored according to the gestational age of the beneficiaries are sent; the duration of call listened to is stored in a database. Both the listenership data and the demographic information is stored in an anonymized manner.

The engagement behaviour of every beneficiary is modelled through an MDP. The binary valued actions correspond to making a service call (denoted by $a=1$, active action) or not making a service call (denoted by $a=0$, passive action); the action set $\mathcal{A} = \{0,1\}$. 
In the ARMMAN setting, we define the states of each beneficiary (arm) based on their recent engagement with the system. Specifically, if a beneficiary listens to at least one automated voice call with more than 30 seconds in a week, the beneficiary is marked as engaging. Thus, $s=0$, corresponds to Non-Engaging (NE) state and $s=1$ corresponds to Engaging (E) state; the set of states $\mathcal{S} = \{0,1\}$. Finally, with 2 states and 2 actions, the Markov chain for every beneficiary can be represented using a 2-state Gilbert-Elliot model \cite{gilbert1960capacity}. The reward function is chosen to maximize the engagement of beneficiaries (i.e., number of engaging beneficiaries) in the long run. Specifically, the reward function for the $n^{th}$ arm/beneficiary is simply defined to be, $R_n(s)=s$ for state $s \in \{0,1\}$.

\section{Methodology} 

The key idea that underlies the analysis is to use the transition data of the beneficiaries observed in the real-world as the ground truth i.e., as a basis of our analysis. Specifically, the observational data reflects the \emph{true} transition probabilities of the arms (beneficiaries) which, ideally-speaking should have been the basis of our decisions. However, as described next, the limited amount of observational data presents itself as a major challenge here.

\paragraph{Computing Missing Observed Transition Probabilities via Clustering.} The observed transitions of each beneficiary $i$ forms an sequence $\langle (s, a, s') \rangle$, where $(s, a, s')$ denotes that the beneficiary $i$ transitioned from the state $s$ to state $s'$ under action $a$ in a particular week; here $s,s',a \in \{0,1\}$. From the observed transitions we can empirically estimate the true transition probabilities of each arm. However, due to the limited number of active interventions, we do not have sufficient observed data to estimate all the transition probabilities for each beneficiary. In particular, the amount of active transition probabilities are limited because most of the beneficiaries (more than $80\%$) do not receive even a single service call during the entire study, due to which we cannot empirically estimate their true active transition probabilities; however, for most beneficiaries, we have sufficient observational data to estimate the passive transition probabilities.

To compute the missing transition probabilities we (i) cluster the beneficiaries based on their observed passive transition probabilities, then (ii) for each resulting cluster, we pool the observed transitions of all the beneficiaries in that cluster, this ensures that we have sufficient data to compute active transition probabilities for each cluster. Then, (iii) the missing active transition probabilities of each beneficiary is assigned to be the active transition probability of the cluster in which that beneficiary lies. This procedure enables us to compute the missing observed transition probabilities.  

Once the missing observed transition probabilities are inferred using the described method, we analyze the errors in predicting transition probabilities. However, we show that the prediction errors are insufficient in explaining the real-world performance. This motivates us to pursue an alternate, more decision-focused evaluation approach. The decision-focused analysis is based on evaluating the errors in computing the top Whittle indices, which is in line with the Whittle index policy. Towards this, we compare/contrast different methods of computing the Whittle indices and show that a definition based on the well-known Spearman's footrule distance~\cite{diaconis1977spearman} is best suited in this context. Using this, we analyze the performance of the RMAB system deployed in the context of ARMMAN. To get a deeper understanding of the performance, we also present a probabilistic analysis to quantify how well is our RMAB system performing as compared to a purely random algorithm as per the proposed metric.






\section{Real-World ARMMAN studies}
\label{section:previous_studies}





The first real-world study in which RMAB system was deployed in the context of ARMMAN was performed by Mate et al. in April 2021. This study tracks a cohort of 23000 beneficiaries for 7 weeks. The cohort was divided into three groups -- round robin, RMAB and the current standard of care (CSOC). In every group, 125 beneficiaries were selected for intervention every week. In the round robin group, beneficiaries are given service call on a first-come first serve basis based on their registration date. In the RMAB group, beneficiaries are chosen for service calls using the Whittle index policy. The current standard of care group received no service call. The study demonstrated that the RMAB-based system resulted in a $\sim30\%$ reduction in the engagement drops as compared to CSOC group.

To rigorously check the efficacy of the RMAB model, several subsequent field tests have been performed. Specifically, we performed two followup studies where different training datasets are used to learn the mapping function from beneficiaries' demographic data to the MDP parameters; the training dataset in the context of ARMMAN simply comprises of (i) the beneficiaries' socio-demographic features, and (ii) their observed transition data from a past field study. In the April 2021 study, a training dataset used to learn the demographic features to the MDP parameter mapping was collected in a study performed in May 2020 where an RMAB system wasn't used to plan the interventions. 
The second study was performed in  January 2022 using a cohort of 44,000 beneficiaries and it went on for $5$ weeks wherein in each week, in each group, 250 beneficiaries were chosen for intervention. The demographic features to transition probability mapping was learnt using trajectories of beneficiaries' behaviour observed in the April 2021 Study, i.e., the training dataset for January 2022 study was April 2021. In the third study, which was performed in May 2022, the data from the January 2022 study was used for training; it tracked a cohort of 15,000 beneficiaries and the budget of intervention each week was 175. The RMAB system performed very well in the April 2021 study -- resulting in more than ~$30\%$ reduction in the engagement drops as compared to the CSOC baseline.
Figure \ref{fig:studies} shows the cumulative engagement drops prevented by RMAB system as compared to the CSOC group adjusted by the number of service calls.
Clearly, a consistent performance of the RMAB system is not observed across the 3 studies, with April 2021 study performing the best, followed by May 2022 and January 2022 studies wherein the cumulative engagement drops prevented by RMAB system weren't significant.
Given this, the goal of the current work is to diagnose and identify the right evaluation method through which we can explain the different performance of the RMAB model across the three studies.

\begin{figure}
    \centering
    \includegraphics[width=0.7\columnwidth]{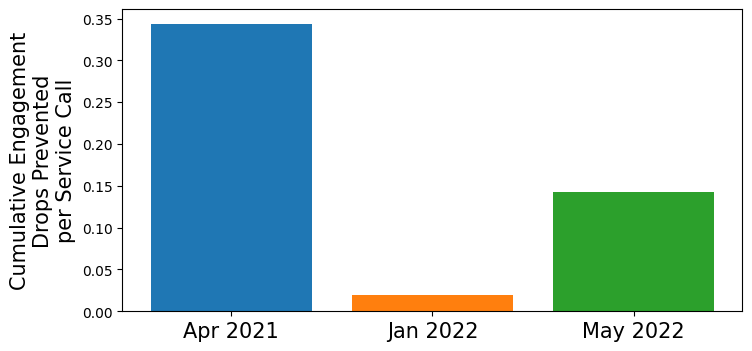}
    \caption{Performance of the RMAB system across the three studies}
    \label{fig:studies}
\end{figure}

\section{Analyzing Errors in the Prediction Stage}
\label{section:transition_probability_errors}

In this section, we provide a comparative analysis of the performance of the RMAB system in the three field tests and we highlight the key learning from these studies. We begin by analyzing the errors in predicting transition probabilities. The main observation we make here is that an overall improvement in the prediction accuracy -- at least as measured by standard metrics such as RMSE or MEA -- may not result in a concomitant improvement in the real-world performance. 

\subsubsection{Comparing Prediction Errors across Studies}
The RMAB-based system in its first phase predicts the MDP parameters of each arm, i.e., the transition probabilities for each beneficiary, denoted by $P_i(s,a,s')$ for each beneficiary $i$ and $s,a,s' \in \{0,1\}$. 
We define the prediction error based on comparing the predicted transition probabilities and the observed transition probabilities for each beneficiary.  We use $O_i(s,a,s')$ to denote the observed (or true) transition probability for each beneficiary. 
Note that, the transition probabilities corresponding to each beneficiary are related by the equations $P_i(s,a,s') = 1 - P_i(s,a, 1-s')$ and $O_i(s,a,s') = 1 - O_i(s,a, 1-s')$ for each  $s,a,s' \in \{1,0\}$. We take this fact into consideration while defining the cumulative error for each beneficiary, the definition of cumulative error is based on only the four independent transition probabilities $P_i(s, a, s)$ where $s,a \in \{0,1\}$.

There are two natural ways of defining the cumulative errors of each beneficiary $i$: first, as the Root Mean-Square Error (RMSE) denoted by $\mathcal{E}_i^{RMSE}$ and second as the Mean Absolute Error (MAE) denoted by $\mathcal{E}_i^{MAE}$. Formally,

\begin{align*}
    \label{equation:cumulative_errors}
    \mathcal{E}^{RMSE}_i = \sqrt{\frac{1}{4} \sum_{s \in \{0,1\}} \sum_{a \in \{0,1\}} \big(P_i(s,a,s) - O_i(s,a,s)\big)^2}
\end{align*}

\begin{align*}
    \mathcal{E}^{MAE}_i = \frac{1}{4} \sum_{s \in \{0,1\}} \sum_{a \in \{0,1\}} |P_i(s,a,s) - O_i(s,a,s)|
\end{align*}

For each beneficiary, we compute the cumulative error values --- $\mathcal{E}^{RMSE}_i$ and $\mathcal{E}^{MAE}_i$ --- as defined above. We then compare the distribution of cumulative errors, and analyze the mean and median error values across the three studies we performed.

\begin{table}[!htbp]
\centering
\caption{Cumulative transition probability prediction errors across studies.}\label{table:cumulative-tp-errors}
\begin{tabular}{|c|c|c|c|}
 \hline
 Error statistics & Apr 2021 & Jan 2022 &  May 2022 \\ 
 \hline \hline
 Mean of $\mathcal{E}^{RMSE}_i$ & 0.382 & 0.451 & 0.345 \\ 
 Median of $\mathcal{E}^{RMSE}_i$ & 0.375 & 0.461 & 0.340 \\
 Mean of $\mathcal{E}^{MAE}_i$ & 0.333 & 0.4 & 0.3 \\ 
 Median of $\mathcal{E}^{MAE}_i$ & 0.311 & 0.394 & 0.291 \\
 \hline

\end{tabular}
\end{table}

In Table \ref{table:cumulative-tp-errors} the mean and the median of cumulative errors, $\mathcal{E}^{RMSE}_i$ and $\mathcal{E}^{MAE}_i$, are shown for the three field studies; the mean and median is computed across all the beneficiaries. Clearly, we can see that the errors in January 2022 study are the highest, wrt to both mean and median. In fact, the cumulative errors are in the following order: errors in January 2022 are the highest, followed by April 2021, and the lowest prediction errors are observed in May 22 study. 

Indeed, this is in direct contrast with the real-world performance of RMAB observed in these studies: as previously illustrated, the performance of RMAB system in May 2022 was much worse as compared to April 2021. Thus, the prediction accuracy, measured by using the standard measures like RMSE and MAE, are not at all indicative of the real-world performance of the RMAB system.

\begin{figure}[!htbp]
\centering
\includegraphics[width=0.8\columnwidth]{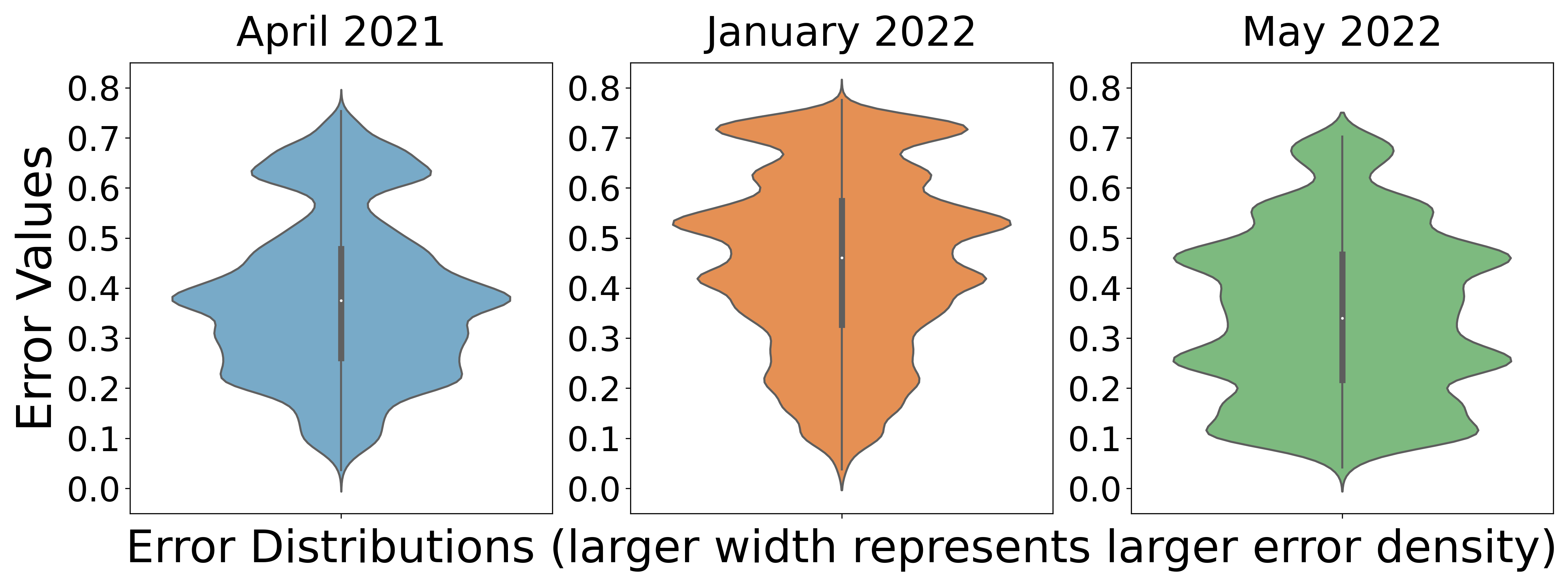}
\caption{The cumulative errors distribution in January 2022 is skewed towards the high-error region, whereas in May 2022 the errors are localized in the low-error region.} \label{figure:cumulative_tp_errors}
\end{figure}

To get more insight, in addition to comparing the mean and median error values, we also analyze the distribution of error values. The distribution of cumulative RMSE errors is shown in Figure \ref{figure:cumulative_tp_errors}. The error distributions, represented as violin plots, depicts the error distribution as a density function. Specifically, error values are shown on the y-axis and the x-axis shows the error density --- the higher the error density for a particular error value, the higher the width of the violin plot at the particular error value\footnote{Since distributions are represented by density function, the area under each distribution is the same}.

Clearly, the distribution of errors is more spread apart in January 2022 -- the errors are concentrated in the high-error regions (error values 0.4 and above). In contrast, the error distribution for the May 2022 study is localized in the low-error region and the distribution tapers in the high-error region (error values 0.5 and above). For the April 2021 study, we can observe two bumps in the error distribution, first, around the error density 0.6 and second around the error density 0.4; the bump in the high-error region is causing the mean and median error values to become higher. Thus, upon visually comparing the error distribution, we arrive at the same conclusion: the prediction errors in January 2022 are the highest, the errors are the lowest in May 2022 study and in April 2021 the errors lie in between the other two studies. Again this is in direct contrast with the performance of the RMAB system in these studies as depicted in Figure \ref{fig:studies} where we see that the performance of RMAB in May 2022 was worse as compared to April 2021.\\

\section{Decision-Focused Evaluation: Analysing Errors in Computing Whittle Indices}
\label{section:decision_focused_analysis}


We now present the decision-focused criterion for analyzing the performance of RMABs. The central idea behind decision-focused evaluation is to analyze the errors in the final decision quality which in turn is captured by the Whittle indices in our formulation. Indeed, this analysis is much more aligned with the Whittle index policy for solving RMABs, as it reveals how inaccurate the RMAB system was in predicting the top Whittle indices.

At every decision step (i.e., every week for ARMMAN), the Whittle index policy selects k beneficiaries having the highest Whittle index (as per their current states). Given this, it is natural to perform the Whittle index error analysis week-by-week. Specifically, for a given study and a given week, we analyze the errors in computing top-k Whittle indices. 

To compute the Whittle index errors, we compare the top-k predicted Whittle indices of beneficiaries to the \emph{observed Whittle indices} -- the Whittle indices of beneficiaries computed using the observed (true) transition probabilities. 

We begin by introducing required notations. For simplicity, we will use the set of first n natural numbers $[n] = \{1, 2, \ldots, n\}$ to denote the set of all beneficiaries; here $n$ denotes the total number of beneficiaries. For a given week, we will use a permutation $P = (b_1, b_2, \ldots, b_n)$ of $[n]$ to denote the ordered sequence of beneficiaries sorted in descending order as per their Whittle indices. 
Similarly, we define $O = (b'_1, b'_2, \ldots, b'_n)$ to denote the ordering of beneficiaries as per their observed Whittle indices in descending order, i.e., $b'_1$ has the highest observed Whittle index in a given week. Furthermore, the integers $\{1, 2, 3, \ldots, k\}$ will be used to denote the $k$ beneficiaries having the highest predicted Whittle indices in a given week, i.e., beneficiary $i$ has the the $i$th highest Whittle index in that week.


\subsection{Quantifying errors in top-k Whittle indices}

Errors in top Whittle indices can be defined in multiple ways and apriori it is not clear which method is the most appropriate. We next present a comparative analysis of various definitions of errors in top-k Whittle indices. We begin by highlighting the shortcomings of the seemingly-natural ways of defining the errors in top-k Whittle indices. Based on this, we select an error definition that doesn't suffer from these shortcomings and we use it to analyze the errors in top-k Whittle indices across the three studies.\\

\noindent
\emph{{\bf 1. Absolute and Normalized Whittle Index Errors:}} Denote by $WI_i^p$ and $WI_i^o$ the predicted and observed Whittle indices of each beneficiary $i$ having top-k predicted Whittle index in a given week i.e. $i \in \{1, 2, \ldots, k\}$. Given this, arguably the simplest way of defining errors in top-k Whittle index is to consider the absolute difference in the predicted and the observed Whittle index for all top-k beneficiaries, denoted by $\mathcal{E}^{abs} = \frac{1}{k} \sum_{i=1}^k |WI^p_i - WI^o_i|$.
However, this seemingly natural definition of errors has an issue: the error values cannot be compared between different studies or different algorithms because the predicted \& observed Whittle indices between two studies can have very different magnitudes as they depend on the cohort of beneficiary in a study; see appendix for supporting data. To mitigate this we can define the normalized error value of each beneficiary $i$ so that it captures the percentage change between the predicted Whittle indices and the observed Whittle indices. 
Formally,

\begin{align*}
    \mathcal{E}^{norm} = \frac{1}{k} \sum_{i=1}^k \frac{|WI^p_i - WI^o_i|}{|WI^p_i|}
\end{align*}

However, this error definition also suffers from a similar issue. We observe that the normalization factor in the denominator i.e., the predicted Whittle indices, $WI^p_i$, depend on the cohort of the beneficiaries, and vary a lot in their magnitudes for different studies, making the error values incomparable across different studies; the details of this observation been deferred to the Appendix.

Notably, a high level issue with both the previous two definitions is that the error values depends on the magnitude of the Whittle indices and not on the ordering, while the Whittle index policy takes decisions solely based on the ordering of Whittle indices. This motivates the use of more sophisticated error definitions that primarily depend on the Whittle index ordering.


\noindent
\emph{{\bf 2. Kendall Tau distance for top-k:}} 
The Kendall Tau distance\cite{diaconis1977spearman} is a well-known metric for quantifying disarray between two permutations. For a given permutation $\pi$, denote by $\pi(i)$ the rank or the index of element $i \in [n]$ in the permutation $\pi$. Let $\sigma$ and $\pi$ be two permutations of the same set of element $[n] = \{1, 2, \ldots, n\}$.  The Kendall Tau distance between $\sigma$ and $\pi$, $\mathcal{K}(\sigma, \pi)$, is defined to be the fraction of discordant pairs between the two permutations, i.e., $\mathcal{K}(\sigma, \tau) = 2/n(n-1) \sum_{1 \leq i < j \leq n} \overline{\mathcal{K}}_{i,j}(\sigma, \tau)$ where $\overline{\mathcal{K}}_{i,j}(\sigma, \tau) = 1$ if the elements $i$ and $j$ are in the same order in both permutations $\sigma$ and $\pi$ and $\overline{\mathcal{K}}_{i,j}(\sigma, \tau) = 0$ otherwise. Formally, $\overline{\mathcal{K}}_{i,j}(\sigma, \pi) = (\sigma(i) > \sigma(j) \land \pi(i) < \pi(j)) \lor (\sigma(i) < \sigma(j) \land \pi(i) > \pi(j))$.

Note that, the Kendall Tau distance quantifies the disarray between relative ordering of elements in two permutations. 
To suit our context, we can modify it to capture the disarray in the top-k elements of $P$, we call this metric the top-k Kendall Tau distance, $\mathcal{K}^{top-k}(\sigma, \tau)$. Let $(s_1, s_2, \ldots, s_k)$ be the top-k elements in the permutation $\sigma$. Given this, the top-k Kendall Tau distance is defined as 

\begin{align*}
    \mathcal{K}^{top-k}(\sigma, \tau) = \frac{2}{n(n-1)} \sum_{1 \leq i < j \leq k} \overline{\mathcal{K}}_{s_i,s_j}(\sigma, \tau)
\end{align*}

where $\sigma$ and $\tau$ are permutations on $k$ elements \footnote{Note that unlike the normalized Kendall Tau distance, the top-k Kendall Tau distance is not symmetric, i.e., $\mathcal{K}^{top-k}(\sigma, \pi) \neq \mathcal{K}^{top-k}(\sigma, \pi)$.}. Based on this definition, we define the errors in top-k Whittle indices as $\mathcal{E}^{kt} = \mathcal{K}^{top-k}(P, O)$. While this error value $\mathcal{E}^{kt}$ can be compared across different studies and different weeks of a given study (because it is normalized by $n(n-1)/2$ and the error values do not depend on the magnitude of Whittle indices), the error definition still has an issue: consider the case when the top-k elements of $P$ appear in the same relative order at the end of the permutation $O$ (i.e., as a suffix); the error value, $\mathcal{E}^{kt} = 0$ for this case. However, in this case, we would want the error to be high since the beneficiaries predicted to have the top-k Whittle indices actually have the lowest observed Whittle indices. In other words, this error definition only captures the relative difference between the ordering of the top-k elements in the orderings $P$ and $O$ -- ignoring their relative positions -- whereas, as highlighted by the previous example, this is insufficient to semantically quantify the Whittle index errors.\\

\noindent
\emph{{\bf 3. Spearman's footrule for top-k:}} We now present a definition of errors in top-k Whittle indices that is a modification of the well-known Spearman's footrule distance \cite{diaconis1977spearman} used to quantify the difference between two permutations. We then show that this definition does not suffer from the shortcomings of the previous definition, and finally compare the errors across different studies wrt this decision-focused error definition. Formally, the Spearman's distance $\mathcal{S}(\sigma, \pi)$ between two permutations $\sigma$ and $\pi$ of the set $[n] = \{1, 2, \ldots, n\}$ is defined as $\mathcal{S}(\sigma, \pi) = \sum_{i=1}^n |\sigma(i) - \pi(i)|$.

To capture the errors in top-k Whittle indices we modify the Spearman's footrule to $(i)$ consider only the top-k elements in $\sigma$ and additionally $(ii)$ we add a normalization factor of $n$. Without loss of generality, let $(s_1, ,s_2 \ldots, s_k)$ be the top-k elements of $\sigma$. Then, the Spearman's footrule for top-k is defined as

\begin{align*}
    \mathcal{S}^{top-k}(\sigma, \pi) & = \frac{1}{k} \sum_{i=1}^k \frac{|\sigma(s_i) - \pi(s_i)|}{n} = \frac{1}{k} \sum_{i=1}^k \frac{|i - \pi(s_i)|}{n} \tag{since $\sigma(s_i) = i$}
\end{align*}

The normalization factor of $n$ is added to the denominator because the quantity $|\sigma(i) - \pi(i)|$ is bounded by the total number of elements/beneficiaries $n$ --- after normalization the error values lie in the interval $[0,1]$ enabling comparison of error values across different studies wherein the number of beneficiaries, $n$, are different. Based on this definition, we define the error in top-k Whittle indices as $\mathcal{E}^{s} = \mathcal{S}^{top-k}(P, O) = \frac{1}{k} \sum_{i=1}^k \frac{|i - O(s_i)|}{n}$ where $s_1, s_2, \ldots, s_k$ are the first $k$ elements of $P$. Additionally, the Whittle index error for each beneficiary $i$ having rank $i$ as per the predicted Whittle index will be denoted by $\mathcal{E}^s_i = \frac{|i - O(s_i)|}{n}$. Therefore,

\begin{align}
    \label{definition:Spearman-WI-error}
    \mathcal{E}^s = \frac{1}{k} \sum_{i=1}^k \mathcal{E}^s_i \ \ \text{ where } 1 \leq i \leq k, \ \ \mathcal{E}^s_i = \frac{|i - O(s_i)|}{n}
\end{align}

This definition quantifies the average shift in the ranks/indices of the top-k beneficiaries (as per ordering $P$) between $P$ and $O$ --- intuitively this is the quantity we want to capture in the error value as the decisions made by RMAB system as based on top-k Whittle indices. Indeed this error definition does not depend on the Whittle index magnitudes and does not suffer from the shortcomings of the previous definitions. 

In the next section, we show that the error values based on the top-k Spearman's footrule distance for the three studies indeed match-up perfectly with the performance of the RMAB-based system observed in the real-world.

\subsection{Errors in Computing top-k Whittle indices}

As previously described, we use the error definition based on the top-k Spearman's footrule, $\mathcal{E}^s = \mathcal{S}^{top-k}(P, O) = \frac{1}{k} \sum_{i=1}^k \frac{|i - O(i)|}{n}$.

For the three studies we compare the errors in top-200 Whittle indices, i.e., we set $k= 200$. This choice is based on the fact that the number of interventions in the April 2021, January 2022, and May 2022 studies are 125, 250, and 175 respectively. Therefore, the value $k = 200$ acts as a middle-ground, enabling fair comparison across studies.

First we will compare the Whittle index error values across the studies, week-by-week, for the first four weeks of all the studies. Then, we will analyze the cumulative errors which combine the weekly error values into a single per-study error value.\\

\noindent
\emph{{\bf (i) Week-by-week comparison of Whittle index errors:}} The mean and median values of the Whittle index errors for the first four weeks of each study are shown in tables \ref{table:mean-WI-errors}, \ref{table:median-WI-errors} respectively.

\begin{table}[!htbp]
\centering
\caption{Week-wise comparison of mean errors in computing top Whittle indices.}\label{table:mean-WI-errors}
\begin{tabular}{|p{2.5cm}|c|c|c|}
 \hline
 Week of the study & Apr 2021 & Jan 2022 &  May 2022 \\ 
 \hline \hline
 Week 1 & 0.424 & 0.505 & 0.492 \\ 
 Week 2 & 0.439 & 0.486 & 0.495 \\ 
 Week 3 & 0.435 & 0.488 & 0.49 \\ 
 Week 4 & 0.446 & 0.502 & 0.465 \\
 \hline
 Cumulative mean errors & 0.436 & 0.495 & 0.486 \\
 \hline
\end{tabular}
\end{table}

\begin{table}[!htbp]
\centering
\caption{Week-wise comparison of median errors in computing top Whittle indices.}\label{table:median-WI-errors}
\begin{tabular}{|p{2.5cm}|c|c|c|}
 \hline
 Week of the study & Apr 2021 & Jan 2022 &  May 2022 \\ 
 \hline \hline
 Week 1 & 0.376 & 0.481 & 0.471 \\ 
 Week 2 & 0.387 & 0.468 & 0.488 \\ 
 Week 3 & 0.389 & 0.485 & 0.477 \\ 
 Week 4 & 0.396 & 0.473 & 0.46 \\
 \hline
 Cumulative median errors & 0.387 & 0.477 & 0.473 \\
 \hline
\end{tabular}
\end{table}

We can observe that the Whittle index error values are the highest for the January 2022 study, followed by May 2022 study and they are the lowest for the April 2021 study. This order of the error values exactly match with the real-world performance of the RMAB system in these studies (Figure \ref{fig:studies}) where the performance of RMAB systems was best in April 2021, followed by May 2022, and January 2022. Therefore, unlike the prediction accuracy analysis, the errors in top-k Whittle indices --- the decision-focused evaluation criterion --- is able to explain the overall performance of the RMAB system in the real-world. In principle, this makes sense because the RMAB system uses top-k Whittle indices to take intervention decisions, while transition probabilities are essentially intermediate values that do not directly influence the decision choices made by RMAB system. \\

\noindent
\emph{{\bf (ii) Comparison of cumulative Whittle index errors:}}
In addition to comparing the errors week-by-week, we can combine the error values across all four weeks for each study to get a per-study error distribution. This distribution of errors is shown in Figure \ref{figure:WI_error_distribution}. Tables \ref{table:mean-WI-errors} and \ref{table:median-WI-errors} list the corresponding values of cumulative mean and median errors. Here also we observe that the errors in April 2021 are lower as compared to the other studies.

\begin{figure}[!htbp]
\centering
\includegraphics[width=0.8\columnwidth]{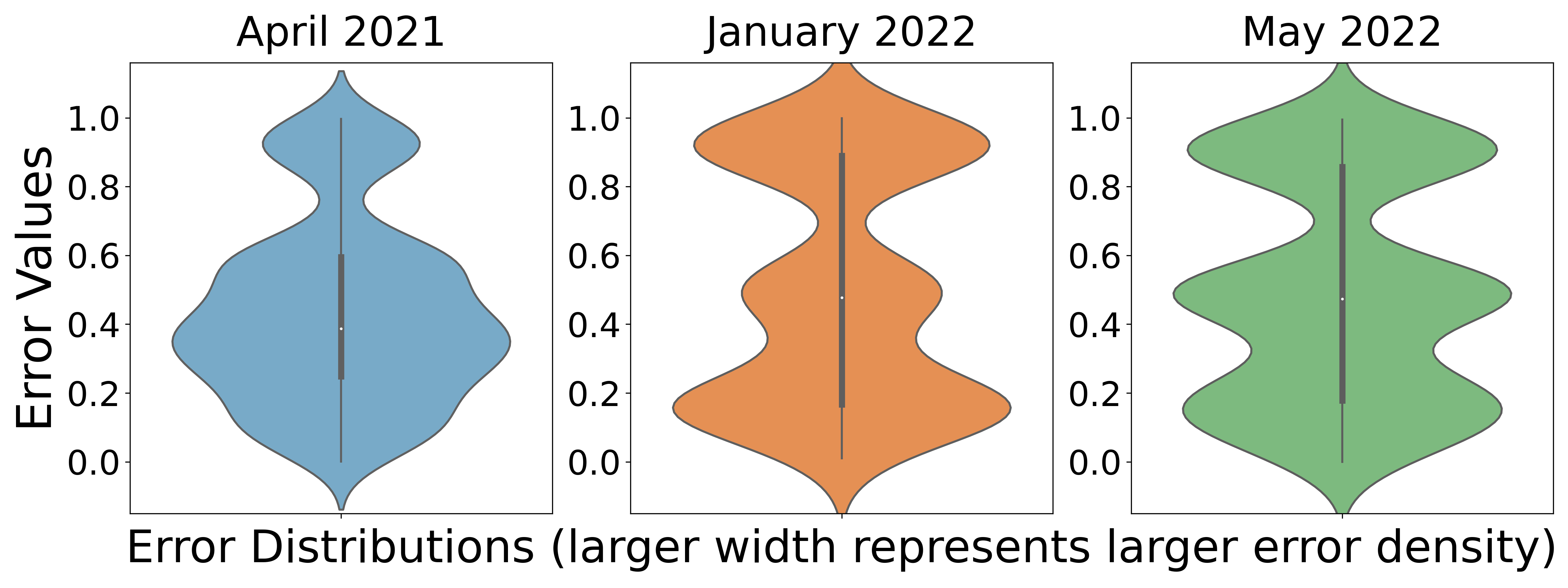}
\caption{The distribution of top-k Whittle index errors is shown for all three studies.} \label{figure:WI_error_distribution}
\end{figure}

\section{Probabilistic Analysis: Comparing RMAB System to a Purely Random Algorithm}
\label{appendix:probabilistic-analaysis}
To theoretically explain the varying performance of RMAB across studies we present a comparison between the RMAB system and a \emph{purely random algorithm} -- an algorithm which selects $k$ beneficiaries uniformly at random (with replacement) for intervention every week -- based on a probabilistic analysis. The analysis is based on computing the expected value and variance of the Whittle index error $\mathcal{E}^s$ (equation \ref{definition:Spearman-WI-error}) for a purely random algorithm. The resulting error values are then used as a baseline to compare the Whittle index error values observed in all the three studies. Thus, this analysis can be interpreted as answering the question: \emph{how well did the RMAB-based system perform as compared to the purely random algorithm?} Due to space limitation the entire analysis has been deferred to the Appendix.

\section{Conclusion}

First we demonstrated that the prediction accuracy analysis alone is insufficient and inconclusive in understanding the  performance of RMAB systems --- this was demonstrated using real-world data of an RMAB system deployed in the context of a maternal health awareness program together with a NGO. We then proposed a decision-focused evaluation method for RMAB systems and showed that this is a far more meaningful evalution to establish correlation with the real-world performance of the RMAB system. Notably, our work primarily focused on RMAB systems, solved using the Whittle index approach. We believe that this analysis will be useful for other deployed applications of RMAB in other domains as well. 






\bibliography{references}

\appendix

\section{Ethics}
We recognize the responsibility associated with deploying real-world AI systems that impacts underserved communities. In our approach, we have iteratively designed, developed and deployed the system in constant coordination with an interdisciplinary team of ARMMAN’s field staff, social work researchers, public health researchers and ethical experts. Particularly, all experiments, field tests and the deployment were performed after obtaining approval from ethics review board at both ARMMAN and Google.

\section{Consent and Data Usage}

The consent for participating in the mMitra program is received from beneficiaries in written form at the time of registration. Additionally, all the data collected through the program is owned by the NGO and only the NGO is allowed to share data. This dataset will never be used by Google for any commercial purposes. The data pipeline only uses anonymized data and no personally identifiable information (PII) is made available to the AI models. The data exchange and use was thus regulated through clearly defined exchange protocols including anonymization, read-access only to researchers, restricted use of the data for research purposes only, and approval by ARMMAN’s ethics review committee. 

\section{Universal Accessibility of Health Information}

To allay further concerns: our system focuses on improving quality of service calls and does not alter, for any beneficiary, the accessibility of health information. All participants will receive the same weekly health information by automated message regardless of whether they are scheduled to receive service calls or not. The service call program does not withhold any information from the participants nor conduct any experimentation on the health information. The health information is always available to all participants, and participants can always request service calls via a free missed call service.

\section{Shortcoming of the Normalized Whittle Index Error}

The normalized Whittle index error definition is capturing the percentage errors between the predicted and the observed Whittle indices. Intuitively this should enable comparison of error values between the studies. However, in trying to use the normalized error definition for the three ARMMAN studies we found an issue: we observed that the top predicted Whittle indices in some studies are lower than the other studies. Due to this the errors values come out to be higher, since the predicted Whittle indices appear in the denominator term. In table \ref{table:range_of_top_WI}, we show the range of top-200 Whittle indices for all the three studies. As we can see, the top-200 Whittle indices in the May 2022 study are much lower than the other two studies, due to this the normalized errors in May 2022 come out to be very high, highlighting a problem in the error definition. Additionally, note that the varying range of Whittle indices show the behavioral differences between the cohort of beneficiaries being considered in different studies, because Whittle indices are dependent on transition dynamics which in turn encode behavioral information of beneficiaries.

\begin{table}[!htbp]
\centering
\caption{Range of top-200 Whittle indices for all the studies.}\label{table:range_of_top_WI}
\begin{tabular}{|p{2.75cm}|c|c|c|}
 \hline
 Week of the study & Apr 2021 & Jan 2022 &  May 2022 \\ 
 \hline \hline
 Highest Whittle index & 0.865 & 1.005 & 0.536 \\ 
 Top 200th Whittle index & 0.836 & 0.868 & 0.360 \\
 \hline
\end{tabular}
\end{table}

\section{Probabilistic Analysis: Comparing RMAB System to a Purely Random Algorithm}
\label{appendix:probabilistic-analaysis}
In the subsequent analysis we compare the expected Whittle index error value of a purely random policy to the error values observed in the three real-world studies. Essentially, this enables us to quantify, how well the RMAB system is performing as compared to a random baseline, thereby offering us quantitative explanation of the difference in the observed performance across studies.

Formally, the purely random algorithm can be modeled as follows: in each week, the algorithm selects a permutation $P$ uniformly at random from the set of all permutations of the elements ${1, 2, \ldots, n}$. Then, it selects the top-k elements having the least rank in $P$ for intervention. The following theorem gives the closed form of the expected Whittle index errors $\mathbb{E}[\mathcal{E}^s]$ and an upper bound on the standard deviation of the errors $\sigma(\mathcal{E}^s)$ for a purely random algorithm. The proof of Theorem \ref{theorem:prob-analysis} has been deferred to the next section.


\begin{restatable}{theorem}{probAnalysis}
\label{theorem:prob-analysis}
The expected value of Whittle index errors for the purely random algorithm is
\[\mathbb{E}[\mathcal{E}^s] = \frac{1}{2} - \frac{k}{2n} + \frac{k^2-1}{3n^2},\]
where $k$ is the number of interventions and $n$ is the total number of beneficiaries. Additionally, the standard deviation of the error value \[\sigma(\mathcal{E}^s) \leq \frac{1}{2\sqrt{3k}}\]
for $k \leq 200$ and $n \geq 3000$.
\end{restatable}

Using Theorem \ref{theorem:prob-analysis} we compare the error values observed in the three studies with the baseline error values of the random algorithm. Specifically, we express the cumulative Whittle index errors of the three studies (shown in table \ref{table:mean-WI-errors}) in terms of $\mathbb{E}[\mathcal{E}^s]$ and  $\sigma(\mathcal{E}^s)$ of the purely random algorithm and quantify how much better is the RMAB system performing. This comparison is shown in table \ref{table:random-policy-comparison}.

\begin{table}[!htbp]
\centering
\caption{Comparing RMAB-based system with purely random policy based on error in top-k Whittle indices.}\label{table:random-policy-comparison}
\begin{tabular}{ | c | c | c | c |} 
 \hline
 Study & $\mathbb{E}[\mathcal{E}^s]$ & $\sigma(\mathcal{E}^s)$ & Cumulative WI error ($\mathcal{E}^s$) \\ 
 \hline \hline
 April 21 & 0.495 & 0.0204 & 0.436 $\leq \mathbb{E}[\mathcal{E}^s] - {\bf 2.892} \ \sigma[\mathcal{E}^s]$ \\ 
 Jan 22 & 0.497 & 0.0204 & 0.495 $\leq \mathbb{E}[\mathcal{E}^s] - {\bf 0.098} \ \sigma[\mathcal{E}^s]$  \\ 
 May 22 & 0.493 & 0.0204 & 0.486 $\leq \mathbb{E}[\mathcal{E}^s] - {\bf 0.343} \ \sigma[\mathcal{E}^s]$  \\
 \hline
\end{tabular}
\end{table}

In April 2021 study, the RMAB-based system performed significantly better than the purely random algorithm -- the error in April 2021 is about 3 standard deviations below the expected error of purely random algorithm. In contrast, we can see that the performance of RMAB-based system in the other two studies is just slight better the purely random algorithm. We can also see that May 2022 is slightly better as compared to January 2022 since in the former the Whittle index errors are 0.343 standard deviation lower than the expected errors of the random algorithm, whereas for January 2022 this number is 0.098. Notably, the relative performance of the RMAB system across studies obtained by the probabilistic analysis exactly matches with real-world performance of the RMAB system as shown in Figure \ref{fig:studies}.


\section{Proof of Theorem \ref{theorem:prob-analysis}}
\label{appendix:missing-proofs}

We begin by proving the following proposition, which will be used in the proof of Theorem \ref{theorem:prob-analysis}.
\begin{restatable}{proposition}{varianceBound}
\label{proposition:variance_bound}
For the purely random algorithm, the following two bounds hold\\ 
1) The variance of the Whittle index error of beneficiary $i$, $Var(\mathcal{E}^s_i) \leq 1/12$.\\
2) The covariance of the errors terms $Cov(\mathcal{E}^s_i, \mathcal{E}^s_j) \leq 0$ where beneficiary $i < j \leq 200$ and $n \geq 3000$.
\end{restatable}
\begin{proof}
1) We begin by considering the variance of error term $\mathcal{E}^s_1$.

\begin{align*}
    \label{equation:prop-proof-0}
    Var(\mathcal{E}^s_1) & = \mathbb{E}[\mathcal{E}^s_1 \mathcal{E}^s_1] - \mathbb{E}[\mathcal{E}^s_1] \cdot \mathbb{E}[\mathcal{E}^s_1] \\
    & = \sum_{i = 1}^n \frac{1}{n} \frac{(i-1)^2}{n^2} - \Big( \sum_{i = 1}^n \frac{1}{n} \frac{i-1}{n} \Big)^2 \\
    & = \frac{(n-1)(2n-1)}{6n^2} - \frac{(n-1)^2}{4n^2} \leq \frac{1}{12} \numberthis
\end{align*}
Using equation \ref{equation:prop-proof-0}, we can upper bound $var(\mathcal{E}^s_i$ for any $i$ as follows

\begin{align*}
    Var(\mathcal{E}^s_i) & = \mathbb{E}[(\mathcal{E}^s_i - \mathbb{E}[\mathcal{E}^s_i])^2] = \mathbb{E}\Big[ \Big(\frac{|j-i|}{n} - \mathbb{E}\Big[\frac{|j-i|}{n}\Big]\Big)^2\Big] \\
    \begin{split} & = \frac{i}{n} \cdot \mathbb{E}\Big[\Big(\frac{|j-i|}{n} - \mathbb{E}\Big[\frac{|j-i|}{n}\Big]\Big)^2\Big | j \leq i \Big] + \\ & \frac{n-i}{n} \cdot \mathbb{E}\Big[\Big(\frac{|j-i|}{n} - \mathbb{E}\Big[\frac{|j-i|}{n}\Big]\Big)^2\Big | j > i \Big] \end{split} \\
    & \leq \frac{i}{n} \cdot \frac{1}{12} + \frac{n-i}{n} \cdot \frac{1}{12} = \frac{1}{12} \tag{using equation \ref{equation:prop-proof-0}}
\end{align*}

\noindent
2) By definition, we can write $Cov(\mathcal{E}^s_i, \mathcal{E}^s_j)$ for beneficiaries $i < j \leq 200$ as

\begin{align*}
    \label{equation:prop-proof-1}
    & Cov(\mathcal{E}^s_i, \mathcal{E}^s_j) = \mathbb{E}[\mathcal{E}^s_i \mathcal{E}^s_j] - \mathbb{E}[\mathcal{E}^s_i] \cdot \mathbb{E}[\mathcal{E}^s_j] \\
    & = \sum_{k \neq l} \frac{|k - i| |l - j|}{n(n-1)} - \sum_k \frac{|k-i|}{n} \cdot \sum_l \frac{|l-j|}{n} \\
    \begin{split} & = \sum_{k \neq l} \Big( \frac{1}{n(n-1)} - \frac{1}{n^2} \Big) |k-i||l-j| \ - \\ & \sum_{k} \frac{1}{n^2} |k-i||k-j| \end{split} \\
    & = \frac{1}{n^2(n-1)} \sum_{k \neq l} |k-i||l-j| - \frac{1}{n^2} \sum_{k} |k-i||k-j| \numberthis \\
\end{align*}
We can upper bound the above expression by replacing $\sum_{k \neq l} |k-i||l-j|$ with $\sum_{k \neq l} |k-i||l-j| \leq \sum_{k \neq l} k \cdot l \leq (n(n-1)/2)^2$. Furthermore, using the fact that $i < j \leq 200$, we get that $\sum_{k} |k-i||k-j| \geq \sum_{k=1}^{n-200} k^2 = \frac{(n-200)(n-199)(2n-399)}{6}$. Combining these observations with equation \ref{equation:prop-proof-1}, we get

\begin{align*}
    \label{equation:prop-proof-2}
    \begin{split} & Cov(\mathcal{E}^s_i, \mathcal{E}^s_j) \leq  \frac{1}{n^2(n-1)} \frac{n^2(n-1)^2}{2^2} \ - \\ & \frac{1}{n^2} \frac{(n-200)(n-199)(2n-399)}{6} \end{split} \\
    \begin{split} & \leq \frac{1}{12n^2(n-1)} \times \\ & \Big( 3n^2(n-1)^2 - 2(n-1)(n-200)(n-199)(2n-399)\Big) \end{split} \numberthis
\end{align*}
Note that the polynomial $3n^2(n-1)^2 - 2(n-1)(n-200)(n-199)(2n-399) = -n^4 + \mathcal{O}(n^3)$. Furthermore, it can be verified that the real roots of the polynomial are $1$ and $2178.45$, therefore, if $n \geq 2179$, then the the value of equation \label{equation:prop-proof-2} is negative. Thus, we have 

\begin{align*}
    \label{equation:prop-proof-3}
    \begin{split}  Cov &(\mathcal{E}^s_i , \mathcal{E}^s_j)  \leq \frac{1}{12n^2(n-1)} \times \\ & \Big( 3n^2(n-1)^2 - 2(n-1)(n-200)(n-199)(2n-399)\Big) \end{split}\\
    & \leq 0 \tag{for $n \geq 3000$}
\end{align*}
This concludes the proof.
\end{proof}

Now we restate and prove Theorem \ref{theorem:prob-analysis}.

\probAnalysis*
\begin{proof}
First we will compute the expected error value and then we will show an upper bound on the variance of the error value for the purely random algorithm.\\

\noindent
\emph{Computing the expected error: } To facilitate mathematical analysis, we view the purely random algorithm as follows: at each decision step, the algorithm selects a permutations of the beneficiaries $P$ uniformly at random and selects the $k$ elements having the least rank for intervention; denote these $k$ beneficiaries by $(1, 2, \ldots, k)$ where beneficiary $i$ has rank $i$.

Using linearity of expectation and equation \ref{definition:Spearman-WI-error}, the expected value of Whittle index error can be written as

\begin{align}
\label{equation:proof-1}
    \mathbb{E}[\mathcal{E}^s] & = \frac{1}{k} \sum_{i=1}^k \mathbb{E}[\mathcal{E}^s_i]
\end{align}

To compute $\mathbb{E}[\mathcal{E}^s]$, we will first compute the value of $\mathbb{E}[\mathcal{E}^s_i]$. Denote by $O = (b_1, b_2, \ldots, b_n)$ the sequence of beneficiaries in descending order of Whittle indices for that week. We know that the purely random algorithm selects $k$ beneficiaries uniformly at random. Thus, we can interpret the algorithm as selecting beneficiary $i$ uniformly at random for each $1 \leq i \leq k$; indeed these $k$ selections are not independent. Using just this interpretation we can compute the value of $\mathbb{E}[\mathcal{E}^s_i]$. Specifically,

\begin{align*}
    \mathbb{E}[\mathcal{E}^s_i] & = \mathbb{E}\Big[\frac{|i - O(i)|}{n}\Big] \tag{via equation \ref{definition:Spearman-WI-error}}\\
    & = \sum_{j=1}^n \frac{1}{n} \cdot \frac{|i - j|}{n} = \sum_{j = 1}^i \frac{1}{n} \cdot \frac{i-j}{n} + \sum_{j=i+1}^n \frac{1}{n} \cdot \frac{j-i}{n}\\
    & = \frac{1}{n^2} \Big({i \choose 2} + {n-i+1 \choose 2}\Big) \numberthis
    \label{equation:proof-2}
\end{align*}
On combining equation \ref{equation:proof-1} with \ref{equation:proof-2} we get

\begin{align*}
    \mathbb{E}[\mathcal{E}^s] & = \frac{1}{k} \sum_{i=1}^k \mathbb{E}[\mathcal{E}^s_i] = \frac{1}{k} \sum_{i=1}^k \frac{1}{n^2} \Big({i \choose 2} + {n-i+1 \choose 2}\Big) \\
    & = \frac{1}{kn^2} \Big( \sum_{i=1}^k {i \choose 2} + \sum_{i=1}^k {n-i+1 \choose 2}\Big) \\
    & = \frac{1}{kn^2} \Big( {k+1 \choose 3} + {n+1 \choose 3} - {n - k + 1 \choose 3} \Big) \tag{both sums telescope} \\
    & = \frac{1}{2} - \frac{k}{2n} + \frac{k^2-1}{3n^2} \tag{simplifying}\\
\end{align*}

\emph{Upper bounding the standard deviation: } Consider the variance of error $\mathcal{E}^s$.
\begin{align*}
    Var(\mathcal{E}^s) & = Var \Big( \frac{1}{k} \sum_{i=1}^k \mathcal{E}^s_i \Big) = \frac{1}{k^2} Var \Big( \sum_{i=1}^k \mathcal{E}^s_i \Big) \tag{via equation \ref{definition:Spearman-WI-error}} \\
    & = \frac{1}{k^2} \Big( \sum_{i=1}^k Var(\mathcal{E}^s_i) + \sum_{i\neq j} Cov(\mathcal{E}^s_i, \mathcal{E}^s_j) \Big) \\
    & \leq \frac{1}{k^2} \Big( \sum_{i=1}^k \frac{1}{12} + 0 \Big)  \leq \frac{1}{12k} \tag{via Proposition \ref{proposition:variance_bound}}\\
\end{align*}

This concludes the proof since $Var(\mathcal{E}^s) \leq 1/12k$ and hence the standard deviation $\sigma(\mathcal{E}^s) \leq 1/2\sqrt{3k}$.
\end{proof}

\end{document}